\documentclass[letterpaper, 10 pt, conference]{ieeeconf}  

\IEEEoverridecommandlockouts                              

\newcommand{\A}{\mathcal{A}}
\newcommand{\C}{\mathcal{C}}
\newcommand{\D}{\mathcal{D}}
\newcommand{\E}{\mathbb{E}}
\newcommand{\calS}{\mathcal{S}}
\newcommand{\M}{\mathcal{M}}
\newcommand{\tr}{\operatorname{Tr}}
\newcommand{\diag}{\operatorname{diag}}
\newcommand{\norm}[1]{\lVert #1\rVert}
\newcommand{\R}{\mathbb{R}}

\newcommand{\train}{\text{train}}
\newcommand{\test}{\text{test}}
\newcommand{\ptr}{P_{\train}} 
\newcommand{\pte}{P_{\test}}
\newcommand{\snr}{\mathrm{SNR}}
\newcommand\numberthis{\addtocounter{equation}{1}\tag{\theequation}}

\newtheorem{remark}{Remark} 
\newtheorem{theorem}{Theorem}
\newtheorem{corollary}[theorem]{Corollary} 
\newtheorem{assumption}{Assumption}[section]
\newtheorem{lemma}[theorem]{Lemma}
\newtheorem{innercustomthm}{Theorem}
\newenvironment{mythm}[1]
    {\renewcommand\theinnercustomthm{#1}\innercustomthm}
    {\endinnercustomthm}

\usepackage{algorithm, algpseudocode}
\usepackage{amsmath, amssymb}
\usepackage{booktabs}
\usepackage{float}
\usepackage{hyperref}
\usepackage{mathtools, multirow}
\usepackage{xcolor}

\title{\LARGE \bf
Data Deletion Can Help in Adaptive RL
}

\author{Param Budhraja\textsuperscript{1}, Aditya Gangrade\textsuperscript{2}, Alex Olshevsky\textsuperscript{$\dagger$,3}, Venkatesh Saligrama\textsuperscript{$\dagger$,4}
\thanks{ Under review for IEEE CDC 2026. 
\textsuperscript{$\dagger$}Equal advising}
\thanks{\textsuperscript{1}Param Budhraja is with Department of Electrical and Computer Engineering,
        Boston University, Boston, MA 02215, USA {\tt \small paramb@bu.edu}}
\thanks{\textsuperscript{2}Aditya Gangrade is with Department of Electrical and Computer Engineering,
        Boston University, Boston, MA 02215, USA {\tt \small gangrade@bu.edu}}
\thanks{\textsuperscript{3}Alex Olshevsky is Faculty of Electrical and Computer Engineering, Systems Engineering, 
        Boston University, Boston, MA 02215, USA {\tt \small alexols@bu.edu}}
\thanks{\textsuperscript{4}Venkatesh Saligrama is Faculty of Electrical and Computer Engineering, Systems Engineering, 
        Boston University, Boston, MA 02215, USA {\tt \small srv@bu.edu}}
}

\begin{document}

\maketitle
\thispagestyle{empty}
\pagestyle{empty}

\begin{abstract}

Deploying reinforcement learning policies in the real world requires adapting to time-varying environments. We study this problem in the contextual Markov Decision Process (cMDP) framework, where a family of environments is indexed by a low-dimensional context unknown at test time. The standard approach decomposes the problem: train a so-called ``universal policy" which assumes knowledge of the true context, then pair it with a context estimator which approximates context using the observed trajectory. We identify a simple, counterintuitive trick that substantially improves the estimator: randomly delete a fraction of the training buffer after each round. This works because data is collected across multiple rounds using progressively better policies, and older trajectories come from a different distribution than what the estimator will face at deployment time; random deletion creates an implicit exponential decay on older data while preserving diversity without requiring any explicit identification of which samples are stale. On classical control and Brax locomotion benchmarks, this reduces robustness gap by 30\% for MLPs and by 6\% on average for recurrent networks. Strikingly, it allows a narrow MLP with 5$\times$ fewer parameters to outperform a wide MLP trained without deletion. To understand when and why deletion helps, we analyze regularized empirical risk minimization with a mismatch between the train distribution and the distribution at deployment; in this idealized setting, we prove that removing a single uniformly random training point decreases expected test loss in expectation under mild conditions. 
For ridge regression we make this quantitative: deletion helps when the regularization coefficient is moderate and the signal-to-noise ratio ($\snr$) is sufficiently low, and, crucially, this $\snr$ threshold gives a direct measure of how large the distribution mismatch between training and deployment must be for deletion to be beneficial.

\end{abstract}

\section{INTRODUCTION}

Reinforcement learning (RL) has been successful in a large diversity of domains, for example, in Atari games \cite{mnih2013playing} 
they have achieved superhuman performance, in the domain of robotics RL algorithms \cite{Lillicrap2015ContinuousCW} 
can achieve returns as good as planning algorithms on the MuJoCo  benchmark \cite{todorov2012mujoco}. However, this does not tell us about the generalization capability of the RL algorithms to unseen yet similar environments, since in the previous works the train and test environments have been the same. 
To study this problem, this work focuses on Adaptive RL. In this setting, the learner is trained on a set of environments drawn from a distribution with the goal of adapting and obtaining a good reward on an unseen test environment drawn from the same distribution. The property of being able to adapt to similar environments is crucial for deploying RL agents in the real world, where the dynamics often depend on fluctuating external conditions, and also for transferring RL policies from simulation to reality. 

Markov Decision Process (MDP) is a powerful modeling tool and can be used to describe dynamic behavior in a broad range of fields. However, dynamics often depend on external factors or variables. To model the effect of these external variables, we use the notion of Contextual MDP (cMDP), introduced in \cite{hallak2015contextual}. In this framework, context is used to refer to external variables affecting the dynamics of the system. Informally, a cMDP is defined as a family of MDPs $\M(c)=(S,A,P_c,R)$ characterized by context $c\in C$. The family shares the state space, action space and rewards but the transition probability matrix depends on context. Our goal is to train a policy $\pi$ that can adapt to $c$ which is unknown at deployment time. 

In \cite{yu2017uposi}, the problem is tackled by decomposing it into two parts: first, given the state and the true context, find the optimal action, and second, given the state-action trajectory, estimate the context. It defines universal policy, a map $K:S\times C \to A$ that, given the true context and state, takes the optimal action. Another way of conceptualizing the universal policy is to think about it as a map $K: C \to \Pi$, where $\Pi$ is the space of policies i.e. given a context universal policy outputs a policy and we will use this definition of $K$ in the rest of the work. The context estimate $\hat{c}_{t} = \phi(s_{t:t-k}, a_{t-1:t-k})$ depends on the state action history. Finally, the adaptive policy is given by $\pi=K(\hat{c}_t)$. One round of training is defined as updating the estimator $\phi$ using the entire available dataset to obtain $\phi'$ and collecting more data using the updated adaptive policy $K(\phi')$. We do multiple such rounds of training. This alternating procedure helps by increasing the diversity of the train data. 

We discover a simple surprising trick that improves the performance of the adaptive policy. 
The trick, called random data deletion, is that we delete a randomly chosen fraction of data. When random data deletion is applied to a multi-round training method it has higher probability of deleting data from older rounds. Intuitively, random data deletion boosts performance since it maintains data diversity while reducing the number of samples from older rounds that are less useful.
To this end, the contributions of this work are as follows:
\begin{enumerate}
    \item Introducing random data deletion which improves the performance of the adaptive policy.
    \item Empirically validate the performance gain of random data deletion for Multi-layer perceptron (MLP) and recurrent neural network (RNN). The experiments show that for MLP we can get improvements upto 30\% and for RNNs we get an improvement of 6\% on average.
    \item Random data deletion helps improve the performance of smaller MLPs. In fact, it helps an MLP with 5x fewer parameters outperform the wider MLP.
    \item 
    We show that, for convex loss functions, random data deletion leads to better test loss when there is a distribution mismatch. When specialized to the case of ridge regression we can quantify the mismatch required.
\end{enumerate}

\section{RELATED WORKS}
To improve the generalization of RL algorithms, recent research has focused on using benchmarks based on procedural generation. In OpenAI ProcGen \cite{cobbe2020leveraging}, the environment is generated by following a pseudo-random procedure starting with a user-defined seed. Other benchmarks with a similar structure are Alchemy \cite{wang2021alchemy}, MazeExplorer \cite{Harries2019MazeExplorerAC}, MetaDrive \cite{li2022metadrive}. These benchmarks share the limitation that the variability in environments is difficult to control, and therefore it is difficult to perform a comprehensive evaluation of the generalizability of RL algorithms. The CARL benchmark used in this paper solves this issue by providing the user with features to control every aspect of the generated environment.

Adaptive RL has goals similar to zero-shot RL where the goal is to train on a set of environments and test on a different set of environments. The difference between two settings is that for Adaptive RL we need a meaningful low dimensional context whereas this is not a strict requirement for zero-shot RL setting. Some methods from the zero-shot RL literature are still relevant and are described here. A common feature of several techniques used to solve the zero-shot RL problem is learning invariant features to generalize to unseen environments. In \cite{sonar2021invariant}, the authors get an invariant policy by learning a policy on top of an invariant representation. In \cite{bertran2020instance}, learn an ensemble of policies on subsets of training domain and then use mean policy of this ensemble to solve unseen tasks. These methods lead to policies that are independent of the variations in the environment while the approach followed by us involves using a policy that can adapt to the different environments. The trade off being that we need to be able to identify the context to take full advantage of the adaptive policy. The authors in \cite{kumar2021rma} take a similar approach where a context estimator is used to condition the policy for solving sim-to-real problem for legged robots. The difference between the proposed approach and their approach is that we directly condition the policy on context while they encode the context and the policy is conditioned on this encoded context. In \cite{zintgraf2020varibad}, a variational auto-encoder is trained to find a latent representation that can be used to predict the past and the future trajectory and the policy is conditioned on this latent representation. While our context estimator simply needs to solve a regression problem their auto-encoder needs to learn the distribution which is much harder but this also makes their method more suitable for general settings where the ground truth context is not available. For a more detailed survey of zero-shot RL research the reader is referred to \cite{kirk2023survey}.

\section{PROBLEM STATEMENT}
In this section, we give a formal definition of the cMDP framework and define the scope of the problem considered in this work. A cMDP is defined to be a tuple $(\C, \calS, \A, \M(c), \D)$ where $\C$ is the context space, $\calS$ is the state space, $\A$ is the action space, $\M$ maps context $c$ to the MDP $\M(c) = (\calS, \A, p(s'|s,a,c), r(s,a), \rho)$ where $p(s'|s,a,c)$ is the probability of transitioning from state $s$ to state $s'$ given you take action $a$ and the context is $c$, the reward function is denoted by $r(s,a)$ and the initial state distribution by $\rho$, and $\D$ is the distribution over context space $\C$. Note that in the above formulation the transition probability depends on context. In general, the reward function $r(\cdot)$ and the initial state distribution $\rho$ may also depend on context. The advantage of this formulation is that it allows us to give precise definition of the train and test set of environments.  

The goal of the learner is, given a set of contexts, to learn a policy that has the best worst-case performance. Furthermore, the learner has access to the ground-truth context during training but not during inference. Formally, let $C_{tr}$ be a set of contexts drawn from the distribution $\D$ of cMDP $(\C, \calS, \A, \M(c), \D)$. The policy trained on this set $C_{tr}$ is denoted by $\pi_{tr}$. Define value function as 
\begin{equation*}
    J^{\pi, c} = E_{\pi, \rho} \left[\sum_{t=0}^\infty \gamma^t r(s_t, a_t) \right],
\end{equation*}
where $\gamma$ is the discount factor, $E_{\pi, \rho} [\cdot]$ indicates that action is taken according to policy $\pi$ and the initial state $s_0$ is drawn from distribution $\rho$. Note that context affects $J^{\pi, c}$ through the transition probability $p(s'|s,a,c)$. The goal of the learner is to find policy $\pi^*$ such that
\begin{equation*}
    \pi^* = \arg \max_{\pi_{tr}} \min_{c\in\C} \; J^{\pi_{tr}, c}.
\end{equation*} 

\section{MAIN METHOD}
In this section, we describe the training procedure for the universal policy and the context estimator. To learn the universal policy, use the augmented state $z_t:=[s_t,c]$ and after every episode of training change the environment by drawing a new context from the training set $C_{tr}$. With these modifications, 
a standard off-the-shelf RL algorithm of one's choice can be used to train universal policy.

\begin{algorithm}
    \caption{Training context estimator with data deletion}
    \begin{algorithmic}
        \State \textbf{Input:} Universal policy $K$, Training context set $C_{tr}$, data retention fraction $\alpha \in (0,1)$.
        \State Initialize data buffer $B$, estimator $\phi_\theta$ with weights $\theta_0$.
        \For{$i=1, 2, \dots,$} 
        \For{$j=1,2,\dots$} 
        \State Choose context $c\in C_{tr}$ and initialize MDP $\M(c)$.
        \If{$i=1$} 
        \State Set $\pi_{use} = K(c)$.
        \Else
        \State Set $\pi_{use} = K(\phi_{\theta_{i-1}})$.
        \EndIf
        \State Collect an episode length trajectory $H\!:=\!\{s_t,a_t\}_{t=1}^T$ on $\M(c)$ using $\pi_{use}$. 
        \State Add $(H,c)$ to data buffer $B$.
        \EndFor
        \State Sample a minibatch of data points $((s_{\tau:\tau-k}, a_{\tau-1:\tau-k}), c)$ from $B$.
        \State Calculate $\theta_{i+1}$ by performing a gradient descent step on $\|\phi_{\theta_i}(s_{\tau:\tau-k}, a_{\tau-1:\tau-k}) - c\|^2$. 
        \State \textcolor{blue}{Delete $1-\alpha$ fraction of trajectories in buffer $B$. 
        }
        \EndFor
    \end{algorithmic}
    \label{alg: train estimator}
\end{algorithm}

However, the universal policy requires true context value which is not available during test time and so we use an estimate of the true context given by the estimator $\phi$. The estimator $\phi$ takes the last $k$ state-actions pairs $(s_{t:t-k}, a_{t-1:t-k})$ to produce the context estimate $\hat{c}_t$. But even the information contained in older state-action pairs $(s_{t-k:0}, a_{t-k:0})$ might be useful in some cases, so a summary statistic of these older state-action pairs could greatly improve the performance of the context estimator. For this purpose, we use recurrent neural networks (RNNs) for context estimation. RNNs typically use the use hidden vector from previous time step to obtain the output. If trained properly the hidden vector should summarize the older state-action pairs. 

The context estimator $\phi$ is learned using a supervised framework. For supervised training, we need trajectories $\{s_t, a_t\}_{t=1}^T$ labeled with the true context $c$ of the environment. 
We will use neural networks for context estimation, so the estimator $\phi_\theta$ will be characterized by its weights $\theta$ and it is trained to minimize the loss
\begin{equation*}
    \min_{\theta} \ell(\phi_{\theta},c) = \sum_{\tau=1}^T \|\phi_{\theta}(s_{\tau:\tau-k}, a_{\tau-1:\tau-k}) - c\|^2,
\end{equation*}
where $s_{t<0},a_{t<0}$ are set to be zero. The loss is simply the norm distance between the true context $c$ and the predicted context $\hat{c}_t = \phi_{\theta}(s_{\tau:\tau-k}, a_{\tau-1:\tau-k})$.
To minimize the loss, we use stochastic gradient descent. A key step in the training procedure is to collect trajectories. In the first round of training we collect trajectories using the policy $\pi = K(c)$, which is the universal policy with the true context, and in the subsequent rounds $i>1$ we use the policy $\pi = K(\phi_{\theta_{i-1}})$, which is the universal policy with the context estimator of the previous round. 
If we train the estimator $\phi$ for a single round, then it will have only seen trajectories where the context in the policy $K(c)$ is the same as the context of the environment $\M(c)$. This is not sufficient since during test time it will inevitably have to deal with the case where the context of the policy $K(\hat{c})$ and the context of the environment $\M(c)$ are different $c\neq \hat{c}$. Thus it is necessary to include trajectories in the training data where there is a mismatch between the context in the policy and that in the environment.

\begin{table} 
    \centering
    \begin{tabular}{|l|c|c|c|c|c|}
        \hline
        Round No. & 1 & 2 & 3 & 4 & 5 \\
        \hline
        Prob. of survival & 1.0 & 0.8 & 0.64 & 0.51 & 0.41 \\
        \hline
    \end{tabular}
    \caption{Probability of survival of a trajectory in random data deletion with $\alpha=0.8$.}
    \label{tab:prob of survival}
    \vspace{-7mm}
\end{table}
We discover that a simple trick can further improve the performance. It is called random data deletion and in this after each round we delete $1-\alpha$ fraction of the stored trajectories, where $\alpha$ is a hyperparameter to be tuned. The data in round $i>1$ is collected using policy $K(\phi_{\theta_{i-1}})$ which is different from the policies of the previous rounds and thus the data comes from a different distribution. So, ideally we would like to train the context estimator using only the most recently collected data but this decreases data diversity and the coverage of the data. Random data deletion puts higher probability on deleting older data. In Table \ref{tab:prob of survival} we present evolution of the probability of survival of a trajectory as the rounds progress. Note that the probability of survival decays exponentially. Thus random data deletion helps with reducing older less useful data while maintaining data diversity.

We investigate this theoretically in section \ref{sec: data del thm}. The analysis considers a simplified setting nevertheless when specialized to the case of ridge regression we find that data deletion helps when we have a distribution mismatch and a moderately valued regularization coefficient.
The pseudocode for the algorithm for training the context estimator is detailed in Algorithm \ref{alg: train estimator}. The implementation can be found at {\small \url{https://github.com/ParamB11/repo-arl}}.

\section{Simulation Results}

\begin{table} 
    \centering
    \begin{tabular}{lcccc}
        \toprule
         & Wide MLP & Narrow MLP & LSTM & GRU \\
        \midrule
        \# params & 47426 & 8322 & 7362 & 5538 \\
        macs & 46976 & 8128 & 7616 & 5760 \\
        \bottomrule
    \end{tabular}
    \caption{Comparison of parameters and MACs of different models used for context estimator.}
    \label{tab:model_comparison}
    \vspace{-10mm}
\end{table}

\begin{table*}
    \centering
    \begin{tabular}{|l|l|c|c|c|c|}
        \hline
         Env & Context & UP(mlp(wide)) & UP(mlp(narrow)) & UP(lstm) & UP(gru) \\
        \hline
        \multirow{3}{*}{Pendulum} 
        & g & 0.80 $\pm$ 0.17 & 0.95 $\pm$ 0.29 & \textbf{0.72 $\pm$ 0.22} & \textbf{0.72 $\pm$ 0.13} \\
        & l & 1.01 $\pm$ 0.29 & 0.87 $\pm$ 0.16 & \textbf{0.84 $\pm$ 0.14} & 0.85 $\pm$ 0.10 \\
        & (g,l) & 1.30 $\pm$ 0.24 & 1.49 $\pm$ 0.23 & 1.24 $\pm$ 0.36 & \textbf{1.11 $\pm$ 0.34} \\
        \hline
        \multirow{3}{*}{LunarLander} 
        & GRAVITY\_Y & \textbf{0.18 $\pm$ 0.03} & 0.32 $\pm$ 0.05 & 0.37 $\pm$ 0.29 & 0.55 $\pm$ 0.60 \\
        & GRAVITY\_X & 1.47 $\pm$ 0.11 & 0.93 $\pm$ 0.14 & 0.24 $\pm$ 0.04 & \textbf{0.21 $\pm$ 0.06} \\
        & G\_X,G\_Y & 1.4 $\pm$ 0.39 & 1.09 $\pm$ 0.55 & \textbf{0.56 $\pm$ 0.09} & 1.28 $\pm$ 0.80 \\
        \hline
        \multirow{3}{*}{Acrobot} 
        & LINK\_MASS\_1 & 0.67 $\pm$ 0.07 & 0.41 $\pm$ 0.09 & \textbf{0.32 $\pm$ 0.09} & 0.46 $\pm$ 0.16 \\
        & LINK\_LEN\_1 & \textbf{0.40 $\pm$ 0.13} & 0.61 $\pm$ 0.10 & 0.69 $\pm$ 0.21 & 0.63 $\pm$ 0.10 \\
        & (LL1, LM1) & 0.70 $\pm$ 0.12 & 0.79 $\pm$ 0.23 & \textbf{0.67 $\pm$ 0.10} & 0.69 $\pm$ 0.17 \\
        \hline
        \multirow{3}{*}{Mt. Car} 
        & gravity & \textbf{0.15 $\pm$ 0.02} & 0.16 $\pm$ 0.02 & \textbf{0.15 $\pm$ 0.03} & 0.18 $\pm$ 0.06 \\
        & force & 0.24 $\pm$ 0.06 & 0.24 $\pm$ 0.04 & \textbf{0.22 $\pm$ 0.03} & 0.24 $\pm$ 0.03 \\
        & (force, gravity) & 0.18 $\pm$ 0.01 & 0.21 $\pm$ 0.05 & \textbf{0.17 $\pm$ 0.10} & 0.18 $\pm$ 0.04 \\
        \hline
        \multirow{2}{*}{Ant} 
        & friction & \textbf{0.16 $\pm$ 0.05} & 0.17 $\pm$ 0.04 & 0.21 $\pm$ 0.03 & 0.21 $\pm$ 0.03 \\
        & gravity & \textbf{0.45 $\pm$ 0.06} & 0.46 $\pm$ 0.09 & 0.51 $\pm$ 0.06 & 0.49 $\pm$ 0.08 \\
        \hline
        \multirow{3}{*}{Halfcheetah} 
        & mass\_torso & \textbf{0.06 $\pm$ 0.01} & \textbf{0.06 $\pm$ 0.01} & 0.08 $\pm$ 0.02 & \textbf{0.06 $\pm$ 0.02} \\
        & friction & \textbf{0.03 $\pm$ 0.00} & \textbf{0.03 $\pm$ 0.01} & 0.04 $\pm$ 0.01 & \textbf{0.03 $\pm$ 0.01} \\
        & (friction, mass\_torso) & 0.14 $\pm$ 0.02 & \textbf{0.13 $\pm$ 0.02} & 0.18 $\pm$ 0.03 & 0.17 $\pm$ 0.02 \\
        \hline
        & Average & 0.55 & 0.52 & \textbf{0.42} & \textbf{0.42} \\
        \hline
    \end{tabular}
    \caption{Comparison of robustness gap of different models across various environments and context parameters.}
    \label{tab:robustness_gap_comparison}
    \vspace{-10mm}
\end{table*}

In this section, we present results that demonstrate the effectiveness of the new training procedure and compare random data deletion with other deletion procedures. For the context estimator, we consider the following architectures: 1. Multi-layer perceptron (MLP) 2. LSTM 3. GRU. 
We consider two MLPs of different configurations: the Wide MLP having hidden layers of width (256,128,64) and the Narrow MLP having hidden layers of width (128, 32, 32). 
For the LSTM and GRU the size of the hidden state is set to 32. In Table \ref{tab:model_comparison} we present the number of parameters and the number of multiply-accumulate (mac) operations required in a forward pass of the architectures considered. The number of mac operations gives a sense of time required for inference in these architectures. The numbers in Table \ref{tab:model_comparison} are reported by fixing the input dimension to 23 and the output dimension to 2. Note that Wide MLP has 5x more parameters compared to other architectures. 

We evaluate these architectures on various Classical control and Box2D environments of the Gym suite. Additionally, we also evaluate them on challenging Brax environments such as Ant and Halfcheetah. We vary context parameters such as gravity, length, mass, etc. in these environments. The evaluation metric used to compare the performance of different architectures is the robustness gap measured with respect to UP-\textit{true}, which is the policy $K(c)$ where $c$ is the true context. The performance of $K(\phi)$ with a perfect estimator should be close to that of UP-\textit{true} and so we report the distance to the performance of UP-\textit{true}. Let $C_{ev}$ be the set of evaluation contexts and $J^*(c)$ be the average episode reward of UP-\textit{true} on MDP $\M(c)$ then the robustness gap of policy $\pi$ is given by 
\begin{equation*}
    \max_{c\in C_{ev}} \frac{J^*(c) - J_{\pi}(c)}{|J^*(c)|},
\end{equation*}
where $J_{\pi}(c)$ is the average episode reward of policy $\pi$ on MDP $\M(c)$. So the lower the robustness gap the closer the performance is to UP-\textit{true} and the better it is.

For training, we sample 300 contexts and use these to train the universal policy and the context estimator. To demonstrate that the algorithm works with any kind of RL algorithm, we use different training algorithms for each environment. For example, for Pendulum we use DDPG to train the universal policy, in LunarLander we use PPO, etc. 
The context estimator is trained for 6 rounds and state-action history of length $k=4$ is used as input. For each training round, we collect 20k new samples. In addition, we ensure that while training recurrent models, we preserve the trajectory structure of the samples. However, while training MLP this is not required. The Table \ref{tab:robustness_gap_comparison} presents the robustness gap of the different architectures considered. The reported values are the mean and the standard deviation of robustness gap over 5 seeds. Also note that the reported values are obtained after tuning the data retention fraction $\alpha$ for these architectures. 
The details of the tuning method are given in subsection \ref{sec: tune data_retain_frac}.
The recurrent models, LSTM and GRU, have the best performance. They even outperform Wide MLP which has 5x more parameters than them.

\begin{remark}[Understanding Table \ref{tab:robustness_gap_comparison}]
    To understand the numbers in Table \ref{tab:robustness_gap_comparison} consider the robustness gap of UP(lstm) and UP(mlp(narrow)) in the LunarLander environment with the context GRAVITY\_X. UP(lstm) has a robustness gap of 0.24, ignoring the standard deviation for now, which means that in the worst case, its episode reward is 24\% less than the best policy, UP-\textit{true}. Similarly, for UP(mlp(narrow)), its worst-case performance is 93\% less than UP-\textit{true}. So, UP(lstm) is 69\% better than UP(mlp(narrow)).
\end{remark}

\subsection{Tuning data retention fraction}
\label{sec: tune data_retain_frac}

In this subsection, we detail the procedure followed to tune the hyperparameter data retention fraction and also highlight the benefits of data deletion. We want to find $\alpha$ that gives the best average robustness gap for each of the architecture environment pairs considered. We restrict the search space of $\alpha$ to the set $\{0.5, 0.8, 0.9, 1.0\}$. In Table \ref{tab:tune data_retain_frac}, we report the robustness gap for UP(lstm) for different environments and $\alpha$ values. Note that the reported values are average over the context variations considered for each environment. Observe that deleting as much as 50\% of the data gives the best performance in 4 out of the 6 environments considered. The column corresponding to the bold value is the optimal $\alpha$ for the environment corresponding to the bold value. 


\begin{table} 
    \centering
    \begin{tabular}{|l|c|c|c|c|}
        \hline
        & $\alpha = 0.5$ & $\alpha = 0.8$ & $\alpha = 0.9$ & $\alpha = 1.0$ \\
        \hline
        Pendulum & 0.95 & 1.06 & 0.94 & \textbf{0.93} \\
        LunarLander & 0.51 & 0.52 & \textbf{0.39} & 0.59 \\
        Acrobot & \textbf{0.56} & \textbf{0.56} & 0.58 & 0.66 \\
        Mt. Car & \textbf{0.18} & 0.2 & 0.2 & 0.25 \\
        Ant & \textbf{0.36} & 0.42 & 0.39 & 0.42 \\
        Halfcheetah  & \textbf{0.1} & \textbf{0.1} & \textbf{0.1} & 0.12\\
        \hline
    \end{tabular}
    \caption{Comparing average robustness gap for UP(lstm) for different values of $\alpha$.}
    \label{tab:tune data_retain_frac}
    \vspace{-5mm}
\end{table}

To demonstrate the benefits of data deletion, we compare the average robustness gap with zero data deletion $\alpha=1.0$ and with some data deleted $\alpha\in \{0.5, 0.8, 0.9\}$. Table \ref{tab:effect of data deletion} presents the robustness gap averaged over all environments for each of the architectures considered. 
Data deletion gives significant benefits for smaller models Narrow MLP, LSTM and GRU. 
In fact, it aids the Narrow MLP, which has 5x less parameters, outperform the Wide MLP.

\begin{table} 
    \centering
    \begin{tabular}{|l|c|c|}
        \hline
        & Best $\alpha \in \{0.5,0.8,0.9\}$ & $\alpha = 1.0$ \\
        \hline
        UP(mlp(wide)) & 0.62 & \textbf{0.57} \\
        UP(mlp(narrow)) & \textbf{0.53}  & 0.82 \\
        UP(lstm) & \textbf{0.43} & 0.5 \\
        UP(gru) & \textbf{0.43} & 0.48 \\
        \hline
    \end{tabular}
    \caption{Effect of data deletion for different architectures}
    \label{tab:effect of data deletion}
    \vspace{-10mm}
\end{table}

\subsection{Ablation: Comparing with alternative deletion methods}
\begin{table} 
    \centering
    \begin{tabular}{lccc}
        \toprule
        & \textbf{stale} & \textbf{random} & \textbf{uniform} \\
        & \textbf{deletion} & \textbf{deletion} & \textbf{deletion} \\
        \midrule
        UP(mlp(narrow)) & 0.75 & 0.94 & \textbf{0.67} \\
        UP(lstm) & 0.74 & \textbf{0.66} & 0.74 \\
        UP(gru) & \textbf{0.65} & 0.66 & 0.91 \\
        \bottomrule
    \end{tabular}
    \caption{Comparing deletion strategies}
    \label{tab:different deletions}
    \vspace{-10mm}
\end{table} 
In this subsection, we compare random data deletion with alternative methods of data deletion. 
In random data deletion as the rounds progress the older trajectories are more likely to be deleted. 
The aim in this section is to compare with other methods that balance emphasis on the old and new trajectories differently as compared to random deletion.

The first method is called stale deletion where we simply delete the oldest trajectories. In the second method, called uniform deletion, we keep a buffer of the all the collected trajectories and we sample a fraction of the trajectories from this buffer to be used for training. 
In uniform deletion trajectories are never deleted but whether they are used for training or not is a stochastic decision. Unlike in the other two methods where trajectories are deleted permanently. 
In stale deletion we do not keep any old trajectories. While in random deletion old trajectories are more likely to be deleted and in uniform deletion the old and the new trajectories are equally likely to be a part of the train data. Note that all the methods have access to roughly the same amount of data.

In Table \ref{tab:different deletions}, we report the average robustness gap over Pendulum and LunarLander. 
Stale deletion is not a clear winner for any of the architectures suggesting that older trajectories are necessary for maintaining the diversity of the train data. Random data deletion gives the very good results for the recurrent architectures, while uniform deletion gives the best result for Narrow MLP. This suggests that different deletion methods might be suited for different architectures.

\section{Data Deletion in a Simplified Setting}
\label{sec: data del thm}
\begin{figure*} 
    \centering
    \begin{tabular}{c c}
        \includegraphics[width=0.3\linewidth]{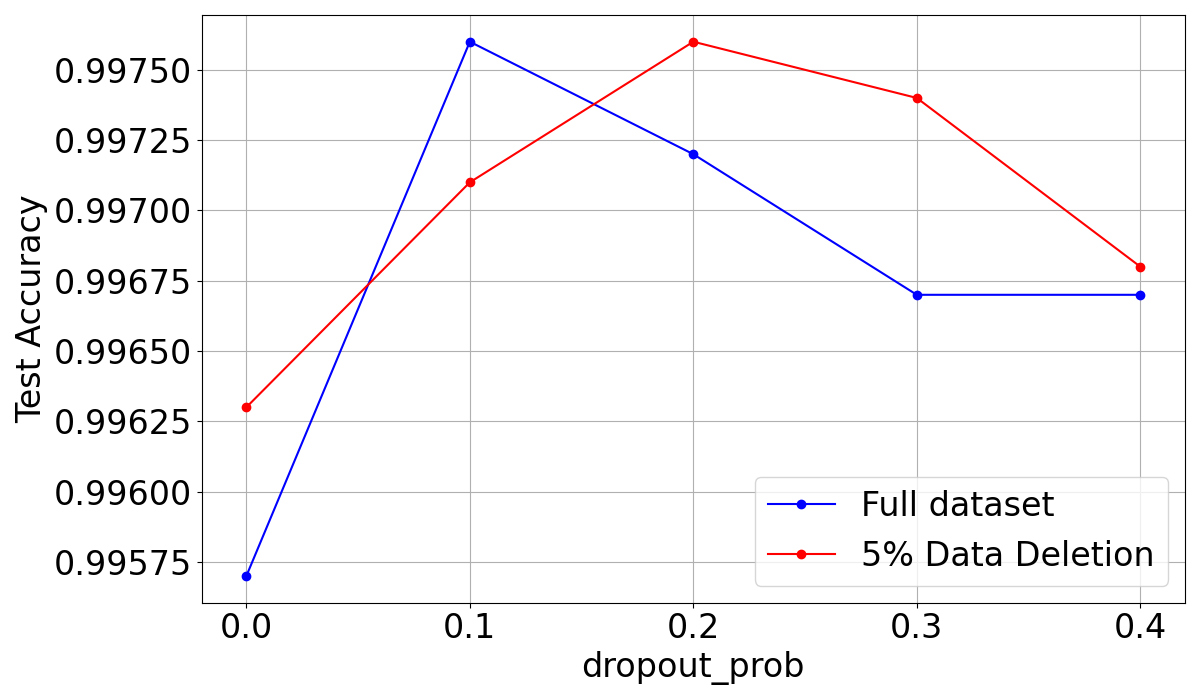} & 
        \includegraphics[width=0.3\linewidth]{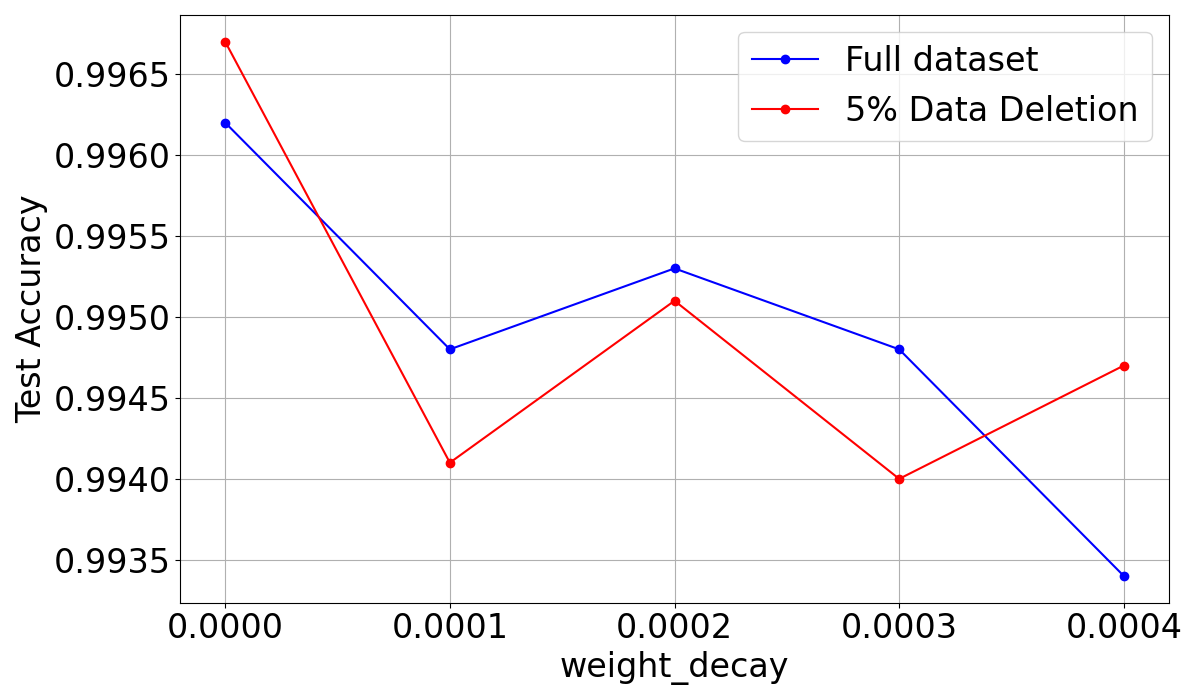}
        \\
        (a) & (b) \\
        \includegraphics[width=0.3\linewidth]{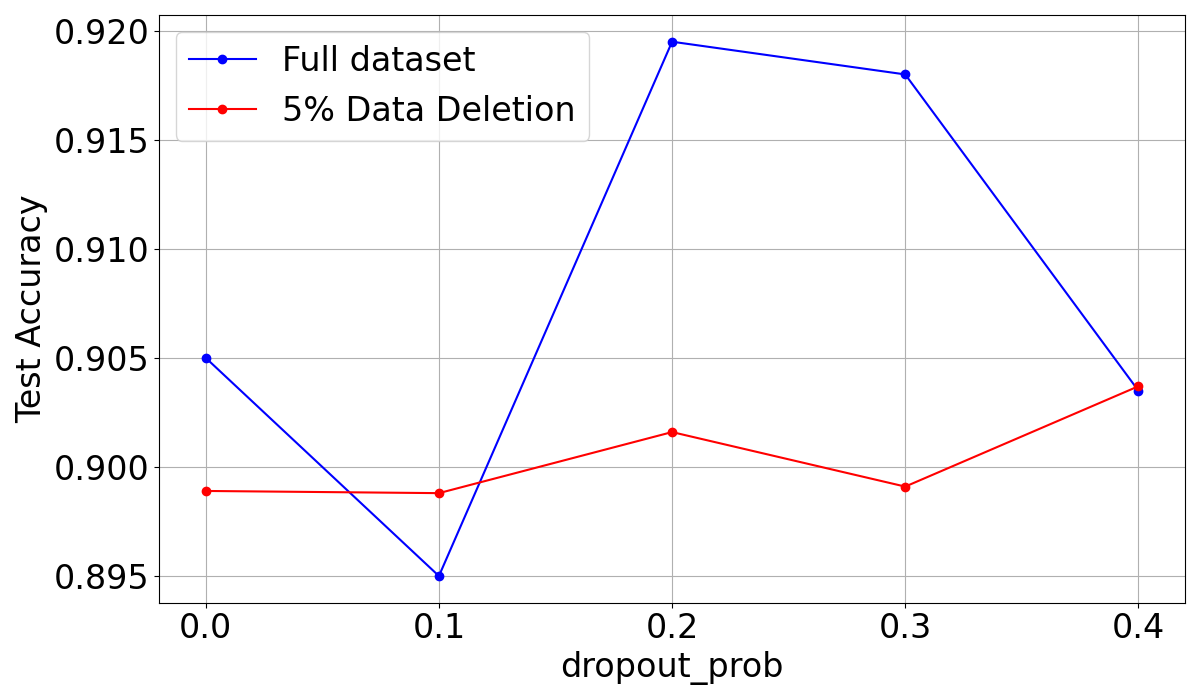} &
        \includegraphics[width=0.3\linewidth]{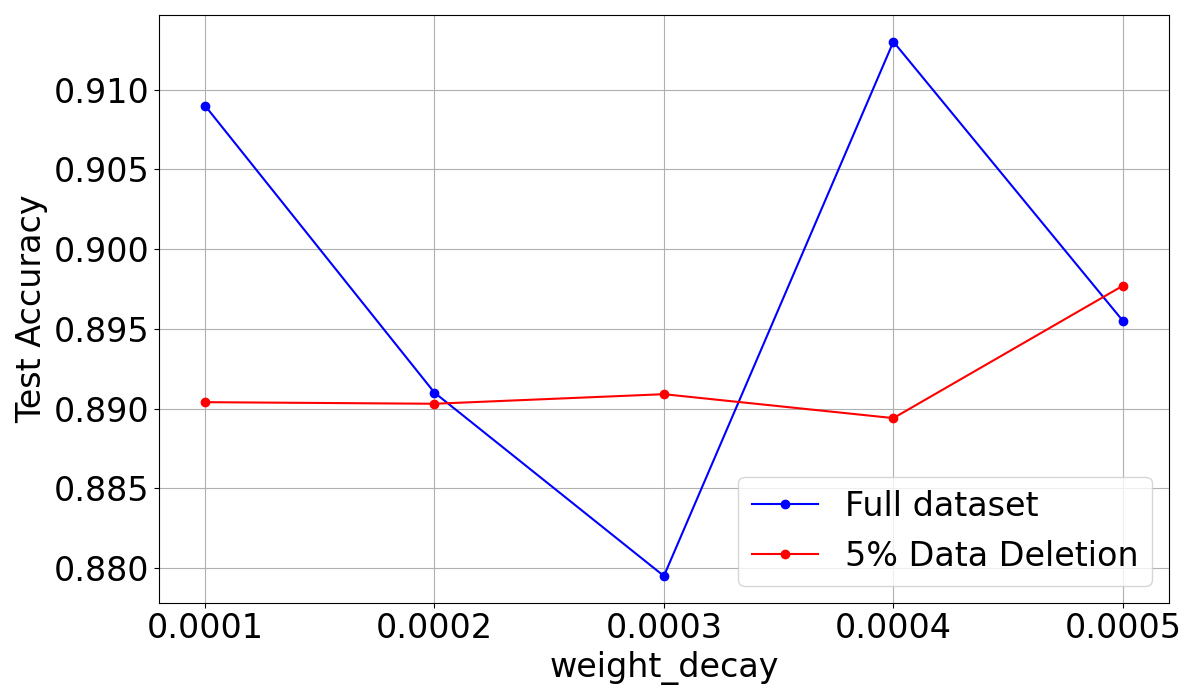} \\
        (c) & (d)
    \end{tabular}
    \caption{Plot test accuracy with an extra class included in the train data. Model trained on full dataset (blue) and model trained with 5\% data excluded (red). Top row plots for MNIST and the bottom row for CIFAR.}
    \label{fig:train neq test}
    \vspace{-5mm}
\end{figure*}

In this section, we study the case of regularized 
risk minimization, where we show that under suitable assumptions random data deletion leads to a lower test error. We do experiments on MNIST and CIFAR dataset which gives additional evidence of the beneficial deletion phenomenon,  Investigating this simple, didactic case provides a better understanding of the phenomenon of data deletion. 
The setting considered is a simple regularized risk minimization problem with the twist that the train data is sampled from the distribution $P_{\train}$ while the test data is sampled from $P_\test$.

\begin{remark}[Comparing with empirical setting]
    The case when the two distributions are different $P_\train \neq P_\test$ is similar to that of training the context estimator. The context estimator $\phi_i$ is trained using data collected in previous rounds which is generated using policy $K(\phi_j)$ where $j<i$. While it is expected to have good performance on data generated by policy $K(\phi_i)$. However the setting is also a simplification since it does not take into account the relation between $P_\train$ and $P_\test$. For the case of context estimator both the $P_\train$ and $P_\test$ depend on previous estimators $\phi_j$ with $j<i$. However in the setting under consideration we simplify it by considering different train and test distributions. Moreover, in the theorem we only consider the case of deleting a single point whereas in practice we consider deleting 10-50\% of data. Also, since we are using neural networks the strong convexity assumption (see Assumption \ref{as:strong convexity} defined later) is typically not satisfied.
\end{remark}
\emph{Setup.} Let $\ptr$ denote a training distribution over a sample space $\mathcal{Z}$: typically, $\mathcal{Z} = \mathcal{X} \times \mathcal{Y},$ where $\mathcal{X}$ is a feature space and $\mathcal{Y}$ is a label space ($[1:K]$ or $\mathbb{R}$). Likewise, we let $\pte$ denote the test distribution. As discussed above, $\ptr \neq \pte$ in general, and the setting of interest is one where we have data drawn from $\ptr$, but wish to evaluate performance on $\pte$. We consider the supervised learning problem over a parametric function class, where we denote the parameters by $w \in \mathbb{R}^d$. Typically, learning over such a class is done by minimising an `empirical risk,' induced by a sample $S = \{z_i\}_{i = 1}^N \overset{\textrm{i.i.d.}}{\sim} \ptr$, i.e., the objective 
\begin{align*}
    F_S(w):=\sum_{i=1}^{N}\ell(w;z_i)+\lambda R(w),
\end{align*}
where $\ell\colon\R^{d}\times\mathcal Z\to\R$ is the loss, $R\colon\R^{d}\to\R$ is the regularizer, and $\lambda > 0$ is the regularisation strength. We let $\hat{w} := \arg\min_w F_S(w)$ denote the optimum. The test risk of a parameter $w$ is denoted $L(w):=\E_{z\sim P_\test}\left[\ell(w;z)\right].$ Our goal is to study the effect of data deletion on $L(\hat{w}).$  To this end, for $i \in [N] = \{1, \cdots, N\}$, we set  \begin{equation}
    F_{S,-i}(w) = F_S(w) - \ell(w;z_i)
\end{equation} to be the objective with data point $i$ deleted, and let $\hat{w}_{-i} := \arg\min_w F_{S,-i}(w).$ After a uniform deletion, the expected test loss is $\mathbb{E}_{i \sim U([N])}[L(\hat{w}_{-i})],$ where $U([N])$ is the uniform law on $[N].$
Our results hold under the following assumptions, which are typically made for such analyses.
\begin{assumption}\label{as:smooth}
    \textbf{(Smooth loss)}  
    Let $\ell(\cdot;z)$ be of class $C^{2}$ and for every $z$ we assume $\norm{\nabla\ell(w;z)}\le G, \norm{\nabla^{2}\ell(w;z)}_{\mathrm{op}} \le H,$
    and the hessian is $\beta$‑Lipschitz, that is,
    \(
          \norm{\nabla^{2}\ell(w;z)-\nabla^{2}\ell(v;z)}_{\mathrm{op}}
          \le\beta\norm{w-v}.
    \)
\end{assumption}

\begin{assumption} \label{as:reg}
    \textbf{(Smooth regulariser)} Let 
    $R$ be of the class $C^{2}$, the gradient satisfies $\norm{\nabla R(w)}\le G_R$ and the hessian satisfies
    $\norm{\nabla^{2}R(w)}_{\mathrm{op}}\le H_R$ 
    and is $\beta_R$‑Lipschitz.
\end{assumption}

\begin{assumption} \label{as:strong convexity}
    \textbf{(Strong convexity)}  
    For every sample~$S$,  
    $F_S$ is $\alpha$‑strongly convex:
    \(\nabla^{2}F_S(w)\succeq\alpha I_d\) for all~$w$.
\end{assumption}

With these in hand, we state our main result, the proof of which is relegated to Appendix \ref{appx:main}.

\begin{theorem}[Beneficial deletion] \label{thm:main}
Let
\begin{align*}
    D &:= \nabla L(\hat{w})^\top H_0^{-1}\left( \frac{-1}{\lambda}\sum_{i=1}^N \nabla\ell\left(\hat{w};z_i\right) \right),
\end{align*} where $H_0 = \nabla^2 F_S(\hat{w})$. Further, define $C:= \frac{3HG^2}{\alpha^2} + \frac{G^3}{\alpha^3} (N \beta + \lambda \beta_R).$ If
\begin{align}
   \lambda D > \frac{CN}{2}, \label{eq:main condn}
\end{align}
then deleting a uniformly chosen training example strictly decreases the test loss, i.e., $\E_{i\sim U([N])}\left[L(\hat w_{-i})\right] < L(\hat w).$

\end{theorem}

\emph{Proof Sketch.} We analyze the effect of deleting a point $i$ by studying the continuous path $F_{S,t,i}(w) = F_{S}(w) -t \ell(w;z_i)$. Notice that $F_{S,1,i} = F_{S,-i}$. Letting $w_i(t) := \arg\min_w F_{S,t,i}(w),$ the test loss along this path is $f_i(t) := L(w(t))$. The quantity $D$ is proportional to $\mathbb{E}_i[ f_i'(0)]$, while the quantity $C$ is a bound on $\sup_{t,i} |f_i''(t)|$. The effect of deletion is then captured by Taylor expanding $\frac1N \sum f_i(t)$ at $0$, and the condition arises by ensuring that the first-order behaviour dominates $C$ appropriately. \hfill $\blacksquare$

Let us now interpret the conditions of Theorem~\ref{thm:main} to better elucidate the result. The term $H_0^{-1} \sum \nabla \ell(\hat{w}; z_i)$ essentially captures the average change in $\hat{w}$ upon deleting a random entry. By demanding that $D$ is large, we are in essence demanding that this change is a descent direction for $L(w)$ at $\hat{w},$ which is a natural condition. The quantitative structure of the bound emerges by controlling the curvature of the loss surfaces, which limits the degree to which each $\hat{w}_{-i}$ moves.  We also point out that the condition of large $D$ requires a distribution shift in order to hold, at least for large $N$. Indeed, if $\pte = \ptr,$ then as $N \to \infty,$ $\frac1N\sum \nabla\ell(\hat{w};z_i) \to \nabla L(\hat{w})$, which forces $D < 0$.


\emph{Application to Ridge Regression.} To further characterize Theorem~\ref{thm:main}, as well as to demonstrate its nontriviality, we next explicitly study its implications on the classical ridge regression setting, where $z = (x,y) \in \mathbb{R}^d \times \mathbb{R}, \ell(w;z) = (w^\top x - y)^2$, and $R(w) = \|w\|_2^2$.
The train samples $(x,y) \sim P_\train$ satisfy $y=w_\star^\top x + n$ where $n\sim \mathcal{N}(0,\sigma^{2})$ and $w_\star$ is the unknown parameter. The test distribution $(x,y) \sim P_\test$, where $y=w_\star^\top x$ and $\E_{x} [xx^\top]=I_d$. We also make the simplifying assumption that the training features satisfy that $\forall i, \|x_i\|_2 = R,$. 
Furthermore, define $X=[x_1, \dots, x_N], M=XX^T$ which has the eigendecomposition $M=U\diag(\mu_j)U^\top$ and $\tilde w = U^\top w_\star$. The main factor causing the mismatch between the train and the test distribution is the noisy labels and a natural way to measure its impact is to use signal-to-noise ratio $\snr\coloneq \frac{R^2\|w_{\star}\|^2}{\sigma^2}$. Our goal will be to give a bound for the $\snr$ for this special case. 
Specializing Theorem~\ref{thm:main} to this setting yields the following sufficient condition, as shown in Appendix~\ref{appx:ridge}. 
\begin{corollary}
    Let $k_1, k_2, k_3 > 0$ be constants satisfying \begin{align}
    \frac{2(1-k_3)}{(1+k_1)^3}  - (1+2k_2) > 0.\label{eq:k1k2k3 condn}
\end{align}
Let $\mu_{\max}$ 
denote the largest eigenvalue of $ M,$ and recall that $\snr := R^2 \|w_*\|^2/\sigma^2$. If $\lambda, \snr$ satisfy  \begin{align}
    \max\left\{\frac{\mu_{\max}}{k_1 R^2}, \right. &\left. \frac{19}{4k_2}, \sqrt{\frac{3}{4k_2}}  \right\} <\frac{\lambda}{R^2} < \frac{k_3 }{\snr}, \label{eq:lambda range} \\
    \snr &< \frac13 \biggl(\frac{2  (1-k_3)}{(1+k_1)^3}  - (1+2k_2)\biggr), \label{eq:snr condn}
\end{align} then uniform data deletion is beneficial in the ridge regression setup described above.\vspace{.2\baselineskip}
\end{corollary}
In other words, data deletion is beneficial in the setting of bounded pointwise $\snr$, there is a range of regularisation wherein deleting a data point uniformly at random improves the test loss. Importantly, we get a lower bound for $\frac{\sigma^2}{R^2 \|w_{\star}\|^2}$ which is precisely the mismatch between the train and the test distribution. 
We note that the condition on $(k_1, k_2,k_3)$ can be explicitly met. For instance, if $k_1=k_3=0.05, k_2=0.1$ satisfy $\frac{2(1-k_3)}{(1+k_1)^3}  - (1+2k_2) = 0.44$.

\subsection{Experiments on Supervised Learning}
Finally, to demonstrate wider occurrence of the beneficial deletion phenomenon, we present experiments that show the practical scenarios where data exclusion improves performance.
We focus on the supervised classification problem. 
We use the MNIST and CIFAR datasets for this experiment. To model presence of out-of-distribution data we include an extra class in the train set which doesn't appear in the test set.
In MNIST, we use a multilayer perceptron, the train set consists of digits $\{0,1,2\}$ and the test set includes digits $\{0,1\}$. In CIFAR, we train a Convolution Neural Network (CNN) on train set with classes $\{\text{airplane, automobile, bird}\}$ and the test set with classes $\{\text{airplane, automobile}\}$. For regularization we use dropouts and weight decay where weight decay refers to adding $\ell_2$ penalty on model weights. We compare the test accuracy of a model trained on the full dataset and model trained with 5\% data excluded. We present results in Figure \ref{fig:train neq test}. 
Observe that for both the datasets and both the choices of regularization data exclusion helps improve test accuracy for at least one regularization value.

\section{Conclusion}
We consider the problem of generalization to unseen environments and use the Contextual MDP framework to formalize this problem. 
We propose a novel procedure to train adaptive policies that simply deletes random samples.
This procedure strikes a good balance between keeping more useful recent data and maintaining data diversity.
The experiments show that random data deletion helps improve the performance of MLP and RNN. It is more useful for smaller neural networks and helps outperform a narrow MLP with 5x less parameters outperform the wide MLP.
Finally, analysis of ERM with convex loss also suggests that data deletion helps when there is a distribution mismatch. 

\bibliography{references}

\appendix
In this appendix we present the proof of Theorem \ref{thm:main} and prove the sufficient conditions for the ridge regression case.
\subsection{Proof of Theorem \ref{thm:main}} \label{appx:main}
%
%
We first show intermediate results as lemmas, and then present the final proof. Recall the notation $F_{S,t,i}$ and $w_i(t)$ from the main text. We first compute the derivative of $w_i(t)$. 
\begin{lemma}
\label{lem:wprime}
Under Assumptions \ref{as:smooth}-\ref{as:strong convexity}, for all $i$,
\begin{equation}\label{eq:wprime}
    w_i'(t) = H_{i,t}^{-1}\nabla\ell\left(w_i(t);z_i\right),
\end{equation}
where $H_{i,t} :=\nabla^{2}F_{S,t,i}\left(w_i(t)\right)$. In particular, at $t=0$, $w_i'(0)= H_0^{-1}g_i,$ where $H_0:=\nabla^{2}F_S(\hat w)$ and $g_i:=\nabla\ell(\hat w;z_i)$.
\end{lemma}

\begin{proof}
Because $F_{S,t,i}$ is $\alpha$‑strongly convex for every $t$, its
minimiser $w_i(t)$ is unique. 
The gradient optimality condition gives us
$\nabla F_{S,t,i}\bigl(w_i(t)\bigr) = 0$.
Differentiating with respect to $t$ 
\begin{align*}
    H_{i,t} w_i'(t) + \partial_t\nabla F_{S,t,i} (w_i(t)) &= 0. \numberthis \label{eq:wprime mid}
\end{align*}
Since $F_{S,t,i}(w) = F_S(w) -t \ell(w;z_i),$ we have $\partial_t \nabla F_{S,t,i}(w)$ $= -\nabla \ell(w;z_i)$. Plugging in $w = w_i(t)$ thus establishes the claim. Of course, $H_{i,0} = H_0$ does not depend on $i$ since $w_i(0) = \hat{w}$. 
\end{proof}
\noindent In the next lemma, we calculate the expectation of the first-order derivative. 
\begin{lemma}[Expected first-order derivative] \label{lem: first der bound}
    The expected first-order derivative is given by
    \begin{equation} \label{eq:Delta1}
        \E_{i\sim U([N])}\left[f_i'(0)\right] = -\frac{\lambda}{N}D,
    \end{equation}
    where $D\coloneq \nabla L(\hat{w})^\top H_0^{-1}\left( \frac{-1}{\lambda}\sum_{i=1}^N \nabla\ell\left(\hat{w};z_i\right) \right)$. 
\end{lemma}
\begin{proof}
    The test loss along the path $f_i(t) = L(w_i(t))$ and taking the derivative with respect to time we get $f'_i(t) = \nabla L(w_i(t))^\top w_i'(t)$. Now using the result from Lemma \ref{lem:wprime} we get $f'_i(t) = \nabla L(w_i(t))^\top H_{i,t}^{-1}\nabla\ell\left(w_i(t);z_i\right)$. Setting $t=0$ and taking expectation yields
    \begin{align*}
        \E_{i\sim U([N])}\left[f_i'(0)\right] &= \frac{1}{N} \sum_{i=1}^N \nabla L(w_i(0))^\top H_0^{-1}\nabla\ell\left(w_i(0);z_i\right) \\
        &= \frac{-\lambda}{N}\nabla L(\hat{w})^{\top}\!H_0^{-1}\!\left(\!\frac{-1}{\lambda} \sum_{i=1}^N \nabla\ell\left(\hat{w};z_i\right)\!\right). \numberthis \label{eq:intermediate first der}
    \end{align*}
    This completes the proof of the theorem.
\end{proof}
In the next lemma, we present a uniform bound for the second derivative $f''_i(t)$. 
\begin{lemma}[Uniform second‑derivative bound]\label{lem:second der bound}
For every $t\in[0,1]$ and every index $i$
\begin{equation}\label{eq:second-bound}
    \bigl|f_i''(t)\bigr| \le
    C,\quad \text{where } \;  C:=\frac{3H\,G^{2}}{\alpha^{2}} + \frac{\underline{\beta}\,G^{3}}{\alpha^{3}}.
\end{equation}
\end{lemma}
\begin{proof}
     In the proof of Lemma \ref{lem: first der bound} we derived $f_i'(t)= \nabla L(w_i(t))^\top w_i'(t)$. Now we differentiate it once more
    \[
        f_i''(t) = w_i'(t)^{\top}\nabla^{2}L(w_i(t))w_i'(t) + \nabla L(w_i(t))^\top w_i''(t).
    \]
    Now we bound each term separately, starting with $w_i'(t)$. Using \eqref{eq:wprime} we obtain
    \begin{align*}
        \norm{w_i'(t)}
        \le \norm{H_{i,t}^{-1}}_{\mathrm{op}}\!\cdot\!\norm{\nabla\ell(w_i(t);z_i)}
        \le \frac{G}{\alpha},
    \end{align*}
    where we have used $\norm{H_{i,t}^{-1}}_{\mathrm{op}}\le 1/\alpha$ which follows from Assumption \ref{as:strong convexity}.
    Next we prove bound for $w_i''(t)$. Using \eqref{eq:wprime mid} and differentiating it with respect to $t$ we get
    \[
        H_{i,t}'\,w_i'(t) + H_{i,t}\,w_i''(t) = \nabla^2\ell(w_i(t);z_i)w_i'(t).
    \]
    Rearranging the above equation to isolate $w_i''(t)$ gives 
    \[
        w_i''(t) = H_{i,t}^{-1}\bigl(\nabla^{2}\ell(w_i(t);z_i)w_i'(t) - H_{i,t}' w_i'(t)\bigr).
    \]
    We bound the norm of each term.
    First, we need to bound $H_{i,t}'$. The chain rule gives
    \(H_{i,t}'=
      -\nabla^{2}\ell(w_i(t);z_i)
      +\nabla^{3}F_{S,t}(w_i(t))[\,w_i'(t)\,]\).
    The operator norm of the third derivative tensor is bounded by the Lipschitz constant of the Hessian, so $\|\nabla^{3}F_{S,t,i}\|_{\mathrm{op}}\le (N-t)\beta+\lambda\beta_R \le \underline{\beta} := N\beta + \lambda \beta_R$ . 
    Hence
    \[
        \norm{H_{i,t}'}_{\mathrm{op}}
        \le
        \norm{\nabla^{2}\ell(w_i(t);z_i)}_{\mathrm{op}}
        +\underline{\beta}\norm{w_i'(t)}
        \le
        H
        +\underline{\beta}\frac{G}{\alpha}.
    \]
    Now we bound the norm of $w_i''(t)$:
    \begin{align*}
        & \norm{w_i''(t)} \\ 
        &\; \le \norm{H_{i,t}^{-1}}_{\mathrm{op}} \left( \norm{\nabla^{2}\ell(w_i(t);z_i)}_{\mathrm{op}}\norm{w_i'(t)} + \norm{H_{i,t}'}_{\mathrm{op}} \norm{w_i'(t)} \right) \\
        &\; \le \frac{1}{\alpha} \left( H \frac{G}{\alpha} + \left(H + \underline{\beta}\frac{G}{\alpha}\right) \frac{G}{\alpha} \right) = \frac{2HG}{\alpha^2} + \frac{\underline{\beta}G^2}{\alpha^3}.
    \end{align*}
    \noindent Finally, we use the above inequalities to bound $f''_i(t)$. 
    Since $\norm{\nabla^2 L(w)}_{\mathrm{op}} \le H$, the first term is bounded by
    \begin{align*}
        \bigl|w_i'(t)^{\!\top}\nabla^{2}L(w_i(t))w_i'(t)\bigr|
        \le H\,\norm{w_i'(t)}^{2}
        \le \dfrac{H\,G^{2}}{\alpha^{2}}.
    \end{align*}
    From Assumption \ref{as:smooth} we have $\norm{\nabla L(w)}\le G$, which helps in bounding the second term by
    \begin{align*}
        &\left|\nabla L(w_i(t))^\top w_i''(t) \right|
        \; \le \norm{\nabla L(w(t))}\,\norm{w_i''(t)} \\
        & \quad \; \le G \left(\frac{2HG}{\alpha^2} + \frac{\underline{\beta}G^2}{\alpha^3}\right) = \frac{2HG^2}{\alpha^2} + \frac{\underline{\beta}G^3}{\alpha^3}.
    \end{align*}
    Adding the bounds gives
    $|f_i''(t)| \le \frac{H G^2}{\alpha^2} + \frac{2HG^2}{\alpha^2} + \frac{\underline{\beta}G^3}{\alpha^3}$, which yields the result in~\eqref{eq:second-bound}.
\end{proof}
Finally, we prove the main theorem next. 
\begin{mythm}{1} 
Suppose Assumptions \ref{as:smooth}-\ref{as:strong convexity} are true and
\begin{align*}
   \lambda D > \frac{CN}{2}, 
\end{align*}
where 
$D, C$ are defined in Lemma \ref{lem: first der bound} and Lemma \ref{lem:second der bound} respectively. Then deleting a uniformly random training example strictly decreases the expected test loss i.e. $\E_{i\sim U([N])}\left[L(\hat w_{-i})\right]
< L(\hat w).$
\end{mythm}
\begin{proof}
Applying Taylor's expansion with the Lagrange form of the remainder to $f_i(t)$ to obtain 
\begin{align*}
    f_i(1) &= f_i(0) + f_i'(0) (1-0) + \frac{1}{2}f''_i(\xi_i) (1-0)^2 \\
    \Rightarrow \Delta_i &= f_i'(0) + \frac{1}{2}f''_i(\xi_i),
\end{align*}
where $\xi_i\in (0,1)$. Taking expectation over $i$ 
\begin{align*}
    \E_{i\sim U([N])}[\Delta_i] = 
    \E_{i\sim U([N])}[f_i'(0)] +\frac12 \E_{i\sim U([N])}\left[f_i''(\xi_i)\right].
\end{align*}
To bound the absolute value of the remainder term we use Lemma~\ref{lem:second der bound} to get
\begin{align*}
    \left| \frac{1}{2} \E_{i\sim U([N])}\left[f_i''(\xi_i)\right] \right| \le \frac{1}{2} \E_{i\sim U([N])}\left[\left|f_i''(\xi_i)\right|\right] \le \frac{C}{2}.
\end{align*}
Using the above inequality with Lemma \ref{lem: first der bound} we obtain
\begin{align*}
    \E_i[\Delta_i] \le
    -\frac{\lambda D}{N} + \frac{C}{2}.
\end{align*}
Under condition $\lambda D > \frac{CN}{2}$, we have $-\frac{\lambda D}{N} + \frac{C}{2} < 0$, so the expected loss strictly drops.
\end{proof}

\subsection{Ridge Regression calculations} \label{appx:ridge}
In this subsection we show that conditions \eqref{eq:lambda range}-\eqref{eq:snr condn} are sufficient for Theorem \ref{thm:main} to hold true. For ridge regression, the expression for $C,D$ is
\begin{align} 
    C &= 2\left[\frac{3R^4}{\lambda^2} \|w_{\star}\|^2 + \sigma^2 \left( \frac{R^2}{\lambda^2} + \frac{19}{4} \frac{R^4}{\lambda^3} + \frac{3}{4}\frac{R^6}{\lambda^4}\right) \right] \label{eq:C ridge} \\
    \label{eq:D ridge}
    D &= 2\sum_{j=1}^d  \frac{\mu_j}{(\mu_j+\lambda)^3} (\sigma^2 - \lambda \tilde{w}_j^2).
\end{align}
The proof of \eqref{eq:C ridge}, \eqref{eq:D ridge} is simple but tedious and is omitted here for the sake of space. Note that \eqref{eq:lambda range} implies that $\lambda < \frac{k_3\sigma^2}{\tilde{w}_j^2}$ for all $j\in[d]$. Using \eqref{eq:lambda range} we get 
\begin{align*}
    D &= 2\sum_{j=1}^d  \frac{\mu_j}{(\mu_j+\lambda)^3} (\sigma^2 - \lambda \tilde{w}_j^2) \\
    &> 2\sum_{j=1}^d  \frac{\mu_j}{(1+k_1)^3 \lambda^3} (1-k_3) \sigma^2\\
    &= \frac{2\sigma^2 \tr(M) (1-k_3)}{(1+k_1)^3 \lambda^3} \eqcolon D'.
\end{align*}
Next we provide an upper bound for $C$
\begin{align*}
    C &= 2\left[\frac{3R^4}{\lambda^2} \|w_{\star}\|^2 + \sigma^2 \left( \frac{R^2}{\lambda^2} + \frac{19}{4} \frac{R^4}{\lambda^3} + \frac{3}{4}\frac{R^6}{\lambda^4}\right) \right] \\
    &< \frac{2R^2}{\lambda^2} [3R^2 \|w_{\star}\|^2 + \sigma^2(1+2k_2)] \eqcolon C'.
\end{align*}
Note that $\lambda D' > \frac{C'N}{2} \Rightarrow \lambda D > \frac{CN}{2}$. Thus 
\begin{align*}
    &\frac{2 \tr(M) (1-k_3)}{(1+k_1)^3 \lambda^2} \sigma^2 > \frac{NR^2}{\lambda^2} [3R^2 \|w_{\star}\|^2 + \sigma^2(1+2k_2)] \\
    \Leftrightarrow\! &\left(\frac{2  (1-k_3)}{(1+k_1)^3}\!\tr(M) - (1+2k_2) NR^2\!\right)\sigma^2 \!>\! 3NR^4 \|w_{\star}\|^2.
\end{align*}
Since the training points lie on the sphere we have $\tr(M)= NR^2$. With this we get
\begin{align*}
    &\left(\frac{2  (1-k_3)}{(1+k_1)^3} NR^2 - (1+2k_2) NR^2\right)\sigma^2 > 3NR^4 \|w_{\star}\|^2 \\
    \Leftrightarrow &\left(\frac{2  (1-k_3)}{(1+k_1)^3} - (1+2k_2)\right) \sigma^2 > 3R^2 \|w_{\star}\|^2 \\
    \Rightarrow &\frac{\sigma^2}{R^2 \|w_{\star}\|^2} > 3\biggl(\frac{2  (1-k_3)}{(1+k_1)^3}  - (1+2k_2)\biggr)^{-1}.
\end{align*}
The last inequality is exactly \eqref{eq:snr condn} and thus the inequality \eqref{eq:main condn} is satisfied and Theorem \ref{thm:main} is true.

\end{document}